\documentclass[11pt]{article}

\usepackage[T1]{fontenc}
\usepackage[utf8]{inputenc}
\usepackage{lmodern}
\usepackage{microtype}
\usepackage{geometry}
\usepackage{graphicx}
\usepackage{tikz}
\usetikzlibrary{arrows.meta,positioning,calc}
\usepackage{xcolor}
\usepackage{booktabs}
\usepackage{amsfonts}
\usepackage{amsmath}
\usepackage{amssymb}
\usepackage{amsthm}
\usepackage{algorithm}
\usepackage{algorithmic}
\usepackage[numbers,sort&compress]{natbib}
\usepackage{url}
\usepackage[colorlinks=true,allcolors=blue]{hyperref}

\newtheorem{definition}{Definition}
\newtheorem{proposition}{Proposition}

\title{Multi-Resolution Attribution from Adaptive Routing State}
\author{Joss Armstrong\\
Ericsson, Athlone, Ireland\\
\small ORCID: 0009-0009-3462-9679}
\date{}

\begin{document}

\maketitle

\begin{abstract}
Adaptive hierarchical systems accumulate routing state as they learn which components to select. We show that this state already defines a coherent attribution over the hierarchy. A leaf receives the product of the local routing weights on its path, while an internal node receives the corresponding prefix product. The same learned state can therefore be read consistently at group and component levels, and every finer readout sums exactly to its coarser counterpart.

This attribution describes the preferences learned by the deployed router rather than an intrinsic or counterfactual value of a component. Across LLM, Census, agentic, and telecom-network hierarchies, the learned state contains meaningful structure at several levels, and the clearest organisation need not occur at the leaves. In the telecom study, Site- or Region-level readouts usually reveal clearer structure than Cell-level readouts. Comparison with Shapley attribution can then show whether the preferences learned in deployment match capabilities revealed by counterfactual coalitions.

The result is a hierarchical explanation that requires no separate attribution model: the same routing state supports consistent explanations at several levels of the system.
\end{abstract}

\section{Introduction}
\label{sec:intro}

Adaptive systems increasingly choose among models, tools, services, or policies through a hierarchy of selectors.
Mixture-of-experts systems route among experts~\cite{shazeer2017outrageously,fedus2022switch}, compound AI systems compose models and tools~\cite{zaharia2024shift}, and automated network-management systems make decisions through several organisational levels.
When such routing is stateful, the system does not merely emit a sequence of choices.
It accumulates a structured record of its own experience in the weights maintained by its selectors.

That state is already an explanation of a particular kind.
It says how the deployed system has learned to distribute preference across the alternatives that were available to it.
Yet attribution is normally introduced as a separate computation after the fact.
A flat importance vector is estimated from component outputs, features, gradients, or counterfactual interventions, while the hierarchy in which the system actually learned is often discarded.

This paper takes the opposite starting point.
If a deployed hierarchy maintains local routing weights, what global explanatory object is already implicit in those weights?
The answer is elementary.
Multiplying the local weights on a root-to-leaf path gives a mass for the leaf.
Stopping the product at an internal node gives the mass of the corresponding subtree.
More generally, any complete cut through the hierarchy yields a distribution, and a finer cut aggregates exactly to every coarser one.
We call the construction \textbf{BOHM}, for \emph{Byproduct Of Hierarchical Metadata}, because the attribution is derived from metadata the hierarchy already maintains rather than from a separate attribution procedure.

This changes the question that attribution answers.
BOHM does not estimate an intrinsic ``importance'' of a component and it is not intended to reproduce a Shapley value.
It reports the adaptive preference state of the deployed hierarchy.
A high mass on a component means that the routing mechanism has learned to place weight there under its own feedback and history.
Whether that preference is good, bad, calibrated, exploitable, or suboptimal is a separate question.

That separation is central to the empirical programme in this paper.
External quality measurements are used to \emph{characterise the routing state} by comparing its ordering with an independently defined performance ordering.
If BOHM mass is strongly aligned with component quality, the router has learned an ordering similar to that quality measure.
If alignment is weak or negative, BOHM has not failed to read the router. The router may have learned a different ordering, may not have accumulated enough evidence to resolve the ordering, or may be evaluated against a quantity that does not coincide with its deployed objective.

Retaining the hierarchy introduces a second question that a flat attribution cannot ask: \emph{at what resolution is the learned structure expressed?}
A state can be coherent at Region level but poorly resolved at Cell level, or vice versa.
This is not a matter of applying a different attribution algorithm at each depth.
The same routing state is being read at different cuts of the same tree.
The distinction matters in practice.
In a four-network cellular study, quality-aligned structure is overwhelmingly expressed above the individual Cell level, even though scale-isolated controls show that Cell-level structure remains visible when it is the only variation present.

The routing history also helps explain why different cuts can look different.
Every routed observation follows one path through the hierarchy, so at any complete cut the same history is partitioned among the local states on that cut.
Finer cuts therefore contain more local weights, some of which may have been updated far less often than their ancestors.
They also use longer path products.
The BOHM readout remains exact, but the learned state being read can be much better resolved at one level than another.

Finally, routing-state attribution is complementary to counterfactual attribution.
Shapley methods~\cite{shapley1953value,lundberg2017unified} ask how a coalition value changes when components are present or absent.
BOHM asks what preference state the deployed hierarchy actually reached.
These objects can be similar when deployed routing already reflects the counterfactual capability of the components.
They can diverge when the deployed router systematically prefers a component that a counterfactual intervention would replace.
That disagreement is informative rather than contradictory.

The paper makes three contributions:

\begin{enumerate}
\item \textbf{A multi-resolution attribution object from routing state.}
We define BOHM on arbitrary finite routing trees and prove exact cut efficiency and refinement consistency.
A single routing state therefore induces an attribution measure over the whole hierarchy rather than one flat vector.

\item \textbf{Resolution as part of the explanation.}
At every complete cut, BOHM masses sum to one and the routed observations can be counted at the same nodes.
The same routing state can be read at different cuts of the hierarchy, but finer cuts depend on more local weights that may have been updated from less experience. This gives a practical reason to treat attribution resolution as part of the explanatory object.

\item \textbf{Empirical characterisation across four settings.}
We examine learned routing state in LLM, Census, telecom-network, and agentic hierarchies.
The studies show that meaningful structure can occur at several resolutions, that leaf-level explanation can be over-resolved, and that comparison with Shapley attribution reveals the difference between deployed preference and counterfactual contribution.
\end{enumerate}

\section{Related work}
\label{sec:related}

\paragraph{Coalition and post-hoc attribution.}
The Shapley value~\cite{shapley1953value} assigns each player its expected marginal contribution over coalitions, and SHAP~\cite{lundberg2017unified} adapts that framework to model explanation.
Exact and approximate algorithms differ in how the coalition value is evaluated or estimated~\cite{chen2022algorithms}.
The same game-theoretic idea has been extended from features to training data and other components, for example Data Shapley~\cite{ghorbani2019data}.
Other post-hoc methods such as LIME~\cite{ribeiro2016lime} and Integrated Gradients~\cite{sundararajan2017axiomatic} explain input features rather than components in a deployed routing hierarchy.
BOHM starts from a different object: persistent selection state already maintained by the system.
Its output is therefore descriptive of deployed preference rather than a decomposition of a counterfactual value function.

\paragraph{Adaptive gating and mixture-of-experts systems.}
Mixtures of local experts learn gates over specialised components~\cite{jacobs1991adaptive}. Modern sparse MoE systems scale that idea to much larger expert sets~\cite{shazeer2017outrageously,fedus2022switch}.
Compound AI systems similarly combine models, retrievers, and tools into larger programs~\cite{zaharia2024shift}.
These systems provide the structural setting for BOHM, but two distinctions matter.
First, BOHM requires persistent routing state across interactions. A standard token-conditional MoE gate is an input-dependent probability vector for one forward pass rather than the cross-round state analysed here.
Second, inspecting one local gate does not itself provide a globally consistent attribution over a hierarchy.
The BOHM path product turns local simplexes into one measure whose cuts are exactly aggregation-consistent.

\paragraph{Attention and internal state as explanation.}
The interpretability literature has long debated when internal model weights or activations should be treated as explanations.
Attention is a familiar example: Jain and Wallace~\cite{jain2019attention} show that attention weights need not correlate with gradient-based importance, while Wiegreffe and Pinter~\cite{wiegreffe2019attention} argue that attention can still be explanatory under an appropriate claim.
The present setting is narrower.
BOHM does not infer causal importance from an arbitrary internal tensor.
It gives semantics to a state whose direct operational role is to allocate selection probability, and it states explicitly that this preference need not coincide with intrinsic or counterfactual component merit.

\paragraph{Hierarchical credit assignment.}
Hierarchical reinforcement learning decomposes decision making over multiple temporal or organisational levels.
Feudal reinforcement learning~\cite{dayan1993feudal}, the options framework~\cite{sutton1999between}, and FeUdal Networks~\cite{vezhnevets2017feudal} address how high-level decisions, subgoals, or temporally extended actions influence lower-level control.
That is a credit-assignment and policy-learning problem.
BOHM instead assumes that a hierarchical selector state already exists and asks what explanatory measure that state induces.
No new policy gradient or reward decomposition is introduced.

\paragraph{Online weighted selection.}
The routing substrate used in the experiments belongs to the broad family of online weighted selection methods.
Weighted Majority~\cite{littlestone1994weighted}, EXP3~\cite{auer2002nonstochastic}, and the multiplicative-weights framework~\cite{arora2012multiplicative} maintain or update preferences over actions using observed feedback.
The specific hierarchical one-bit substrate used here is analysed separately in~\cite{josarm2026implicit}, including its stationary single-selector equilibrium and hierarchical composition properties.
BOHM's contribution is orthogonal to the update rule: any routing tree with local non-negative simplexes induces the node masses and cut identities in Section~\ref{sec:method}.
The equilibrium result is used only to interpret the particular state-generation mechanism employed by the experiments.

\paragraph{Positioning.}
The closest conceptual neighbours therefore address different layers of the problem.
Coalition methods explain counterfactual contribution, online-learning methods generate routing preferences, and hierarchical-control methods learn or assign credit across decision levels.
BOHM formalises the explanatory object already implicit in the resulting routing state.
Its distinctive feature is not a new optimiser but the exact composition of local preference state into a single measure that can be read consistently at several levels.

\section{Routing state as attribution}
\label{sec:method}

\subsection{Hierarchical routing state}
\label{sec:setting}

Let $\mathcal{T}$ be a finite rooted tree.
Components occupy its leaves.
Every selector node $v$ with children $1,\ldots,b_v$ stores a local weight vector
\[
\mathbf{w}_v(t)=
\bigl(w_{v,1}(t),\ldots,w_{v,b_v}(t)\bigr),
\qquad
w_{v,i}(t)\ge0,
\qquad
\sum_i w_{v,i}(t)=1.
\]
Degree-one internal nodes contain no decision and are treated as identity edges of weight one.

The weights are the persistent adaptive state of the hierarchy.
They need not be optimal, stationary, or aligned with any external quality measure for BOHM to be defined.
The experiments generate them with the routing substrate in Appendix~\ref{app:substrate}, but the attribution construction below depends only on the simplex state itself.
Accordingly, the cut-efficiency and refinement-consistency results require no assumption about how the weights were learned. The substrate-specific equilibrium result in Section~\ref{sec:local_equilibrium} is used only to interpret the experimental state.

\subsection{Node mass and cuts}

\begin{definition}[BOHM node mass]
For any node $u$, let $P(u)$ be the sequence of edges from the root to $u$.
The BOHM mass of $u$ is
\begin{equation}
a_u(t)=\prod_{e\in P(u)} w_e(t),
\label{eq:node_mass}
\end{equation}
with $a_{\mathrm{root}}(t)=1$.
\end{definition}

For a leaf, Eq.~\ref{eq:node_mass} gives the product of all routing weights on its path.
For an internal node, it gives the mass assigned to the corresponding subtree before that mass is divided among its descendants.

\begin{definition}[Resolution cut]
A \emph{resolution cut} $C$ is a set of nodes that intersects every root-to-leaf path exactly once.
The BOHM attribution at resolution $C$ is the vector
\[
\mathbf{a}_C(t)=\{a_u(t):u\in C\}.
\]
\end{definition}

A depth level in a balanced tree is one possible cut, but the definition also covers irregular hierarchies in which different branches terminate or aggregate at different depths.

\begin{proposition}[Cut efficiency]
\label{prop:cut_efficiency}
For every resolution cut $C$,
\begin{equation}
\sum_{u\in C} a_u(t)=1.
\label{eq:cut_efficiency}
\end{equation}
\end{proposition}

\begin{proof}
At each selector, the child weights partition the mass entering the selector because they sum to one.
Applying this identity recursively until the paths reach the cut partitions the root mass of one among the nodes in $C$.
\end{proof}

\begin{proposition}[Exact refinement consistency]
\label{prop:refinement}
Let $C'$ refine $C$.
For every $u\in C$,
\begin{equation}
a_u(t)
=
\sum_{\substack{v\in C'\\v\text{ descendant of }u}}
a_v(t).
\label{eq:refinement}
\end{equation}
\end{proposition}

\begin{proof}
The descendant nodes of $C'$ partition the subtree rooted at $u$.
Repeated application of the local simplex identity therefore partitions $a_u(t)$ among those descendants.
\end{proof}

\begin{figure}[t]
\centering
\begin{tikzpicture}[
  node distance=8mm and 17mm,
  box/.style={draw,rounded corners,minimum width=18mm,minimum height=8mm,align=center,font=\small},
  leaf/.style={draw,circle,minimum size=8mm,inner sep=0pt,font=\scriptsize},
  edge/.style={-Latex,thick},
  cut/.style={dashed,gray!70,thick}
]
\node[box] (r) {root};
\node[box,below left=of r] (a) {$A$\\$a_A=0.60$};
\node[box,below right=of r] (b) {$B$\\$a_B=0.40$};
\node[leaf,below left=of a] (a1) {$A_1$};
\node[leaf,below right=of a] (a2) {$A_2$};
\node[leaf,below left=of b] (b1) {$B_1$};
\node[leaf,below right=of b] (b2) {$B_2$};
\draw[edge] (r)--node[left,font=\scriptsize]{$0.60$}(a);
\draw[edge] (r)--node[right,font=\scriptsize]{$0.40$}(b);
\draw[edge] (a)--node[left,font=\scriptsize]{$0.70$}(a1);
\draw[edge] (a)--node[right,font=\scriptsize]{$0.30$}(a2);
\draw[edge] (b)--node[left,font=\scriptsize]{$0.25$}(b1);
\draw[edge] (b)--node[right,font=\scriptsize]{$0.75$}(b2);
\node[font=\scriptsize,below=2mm of a1] {$a=0.42$};
\node[font=\scriptsize,below=2mm of a2] {$a=0.18$};
\node[font=\scriptsize,below=2mm of b1] {$a=0.10$};
\node[font=\scriptsize,below=2mm of b2] {$a=0.30$};
\draw[cut] ($(a.north west)+(-4mm,5mm)$) -- ($(b.north east)+(4mm,5mm)$)
 node[pos=1,right,font=\scriptsize,black]{coarse cut};
\draw[cut] ($(a1.south west)+(-4mm,-7mm)$) -- ($(b2.south east)+(4mm,-7mm)$)
 node[pos=1,right,font=\scriptsize,black]{fine cut};
\node[font=\scriptsize,align=center,below=15mm of $(a2)!0.5!(b1)$] {Exact refinement consistency:\\$0.42+0.18=0.60$,\quad $0.10+0.30=0.40$};
\end{tikzpicture}
\caption{One routing state read at two resolutions. Local weights multiply along each path to give node mass. The fine leaf cut sums exactly to the coarser $\{A,B\}$ cut, so changing resolution does not invoke a second attribution procedure.}
\label{fig:bohm_schematic}
\end{figure}

Proposition~\ref{prop:refinement} is the central structural property.
Changing resolution does not recompute attribution under a new method.
It changes only the cut at which the same hierarchical measure is read.
No information from a coarser cut is lost when the hierarchy is retained.

\subsection{What BOHM explains}
\label{sec:semantics}

BOHM is an exact description of a routing state.
That statement is independent of whether the state is desirable.
To avoid conflating these questions, we distinguish three objects:

\begin{enumerate}
\item \textbf{Routing-state attribution:} the BOHM masses $\mathbf{a}_C(t)$.
\item \textbf{External state diagnostics:} quantities such as Kendall correlation between BOHM mass and an independently defined quality ordering.
\item \textbf{Counterfactual attribution:} quantities such as a Shapley value computed from outcomes under alternative component coalitions.
\end{enumerate}

Only the first is BOHM itself.
The second asks what the router has learned relative to some external reference.
The third asks a different counterfactual question.

This distinction matters especially when BOHM and an external quality ordering disagree.
A low rank correlation does not mean Eq.~\ref{eq:node_mass} inaccurately represents the router.
It means that the router's learned preference state is not strongly ordered by that external quantity at the chosen resolution.

The same qualification applies to comparisons with Shapley attribution.
Agreement can be useful evidence that deployed preference and counterfactual contribution happen to align.
Disagreement can reveal a deployed router that is ignoring, underusing, or simply valuing components differently from a coalition-based analysis.

\subsection{What a local weight means in the experimental substrate}
\label{sec:local_equilibrium}

BOHM itself requires only a hierarchy of local simplexes.
The experiments generate those simplexes with the outcome-driven routing substrate in Appendix~\ref{app:substrate}.
Under that substrate, a local weight has a useful stationary interpretation.
For a selector with $b\ge2$ children of stationary qualities $p_1,\ldots,p_b$, the unique interior equilibrium derived in~\cite{josarm2026implicit} is
\begin{equation}
w_i^* = \frac{p_i+c}{1+c},
\qquad
c=\frac{1-\sum_{j=1}^{b}p_j}{b-1},
\label{eq:local_equilibrium}
\end{equation}
under the stated interiority condition.
Consequently,
\begin{equation}
p_i>p_j \;\Longrightarrow\; w_i^*>w_j^*,
\qquad
p_i=p_j \;\Longrightarrow\; w_i^*=w_j^*.
\label{eq:sibling_order}
\end{equation}

Equation~\ref{eq:sibling_order} is deliberately local.
At equilibrium, better siblings receive larger local weights under the experimental substrate.
It does not imply a global quality ordering across branches, because a leaf mass also includes the upstream branch weights on its path.
This is why external quality alignment is treated as an empirical diagnostic of the learned hierarchy rather than as a property built into BOHM.

\subsection{Resolution and local support}
\label{sec:evidence}

Let $N_u(T)$ be the number of the first $T$ routed observations whose path passes through node $u$. Every routed path intersects a complete cut once, so
\begin{equation}
\sum_{u\in C} N_u(T)=T
\label{eq:evidence_conservation}
\end{equation}
for every resolution cut $C$. This identity is useful here for a narrow reason: it tells us how much routing history reached the local states from which a particular readout is formed. At finer cuts the same $T$ routed observations are distributed across more nodes, so some local weights can be based on much less experience than their ancestors. We use the realised counts $N_u(T)$ as a support diagnostic rather than treating Eq.~\ref{eq:evidence_conservation} as a general theorem about optimal resolution.

Fine attribution also contains more estimated factors.
If the local state is observed with multiplicative errors
$\widehat{w}_e=w_e(1+r_e)$, then
\begin{equation}
\frac{\widehat a_u}{a_u}
=
\prod_{e\in P(u)}(1+r_e),
\qquad
\log\widehat a_u-\log a_u
=
\sum_{e\in P(u)}\log(1+r_e).
\label{eq:path_error}
\end{equation}
Thus deeper readouts can inherit uncertainty from more local states.

Together, Eqs.~\ref{eq:evidence_conservation} and~\ref{eq:path_error} motivate a practical question for BOHM: at which level does the learned state contain structure that is sufficiently supported to interpret? The attribution mapping is exact at every cut. The empirical question concerns the state that has been learned at that cut.

\section{What structure is present in the learned state?}
\label{sec:empirical}

The experiments in this section characterise routing state rather than validate an attribution target.
Rank correlation is used when an external ordering is available.
It should be read as a property of the learned state at a chosen cut.

\subsection{A stationary LLM hierarchy}
\label{sec:llm}

We begin with a stationary setting in which the external performance ordering is easy to interpret.
Eighteen LLMs are arranged in a three-level $[3,3,2]$ hierarchy: three broad quality tiers, each divided into three subgroups of two models.
The models are evaluated on 880 LiveCodeBench problems~\cite{jain2024livecodebench}.
Empirical pass rates range from 6.8\% to 80.0\%.

The hierarchy in this first experiment is intentionally constructed from LiveCodeBench quality tiers.
Its purpose is not to demonstrate that BOHM can discover a hierarchy that was hidden from it.
Rather, it provides a clean calibration setting in which the routing feedback and the external ordering refer to the same stationary performance signal.
On each problem the stateful wrapper routes to one model, observes only that model's binary outcome, and updates the local routing weights.
Thus a single routed trajectory observes 880 model--problem outcomes rather than the complete $880\times18$ correctness matrix.
We run 20 routing seeds.

Table~\ref{tab:llm_tiers} summarises the hierarchy and the final mass assigned to each broad tier.
The strongest tier receives about two thirds of the final routing-state mass, while the weakest receives about one eighth.
This is already a hierarchical statement rather than merely a leaf ranking: the state records both broad tier preference and the distribution within each tier.

\begin{table}[t]
\centering
\caption{Broad tiers in the 18-model LiveCodeBench hierarchy. ``Mean BOHM mass'' is the mean final mass of the tier across 20 routing seeds.}
\label{tab:llm_tiers}
\begin{tabular}{lcc}
\toprule
Tier & Pass-rate range & Mean BOHM mass \\
\midrule
A (strong) & 69.9--80.0\% & 66.7\% \\
B (mid)    & 31.4--58.4\% & 20.5\% \\
C (weak)   & 6.8--30.3\%  & 12.8\% \\
\bottomrule
\end{tabular}
\end{table}

\begin{figure}[t]
\centering
\includegraphics[width=0.88\linewidth]{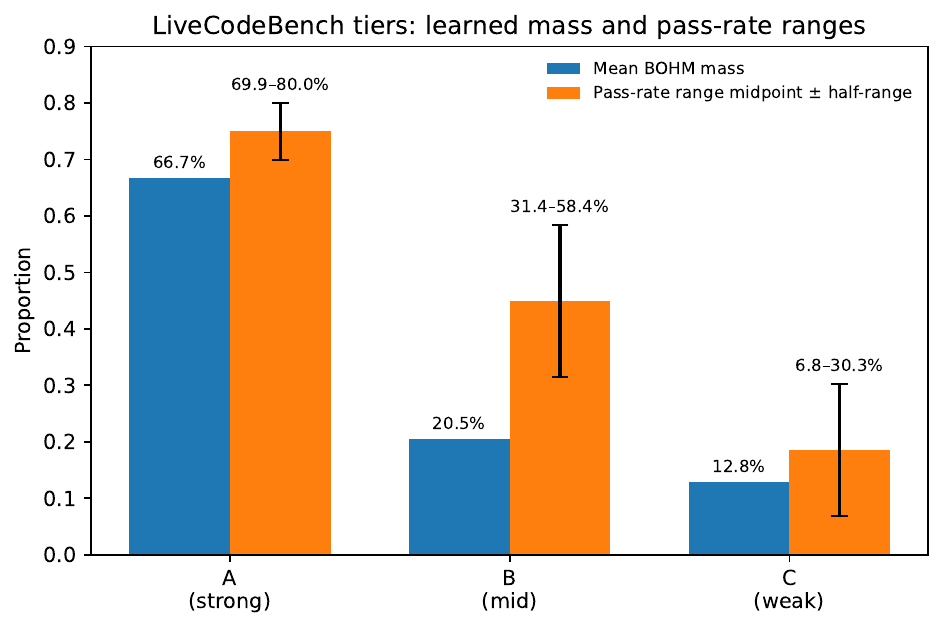}
\caption{LiveCodeBench tiers: mean BOHM mass by broad quality tier, with the corresponding pass-rate ranges. The learned state concentrates most strongly on the top tier but still retains non-trivial mass at intermediate and weak tiers.}
\label{fig:llm_tiers}
\end{figure}

Tier A contains GPT-oss-120B, Qwen3-32B, MiniMax-M2.5, DeepSeek-V3.2, Qwen3-Coder-480B, and DeepSeek-R1-32B.
Tier B contains GLM-4.7-Flash, Qwen2.5-Coder-32B, Qwen2.5-72B (base), Qwen2.5-32B-Instruct, Qwen2.5-14B-Instruct-1M, and Phi-4-14B.
Tier C contains Qwen2.5-14B (base), Qwen2.5-Coder-7B, LLaMA-3.1-70B, DeepSeek-Coder-V2, LLaMA-3.1-8B, and Mistral-7B.

At leaf level, individual learned states have mean Kendall alignment
\[
\tau = 0.739\pm0.079
\]
with empirical model pass rate, with mean Spearman $\rho=0.886\pm0.061$ on the same runs.
Averaging the final BOHM masses across the 20 seeds gives Kendall $\tau=0.928$.
These values do not measure whether BOHM has correctly read the routing state: that readout is fixed by Eq.~\ref{eq:node_mass}.
They show instead that, under a stationary feedback signal closely matched to the external quality ordering, the learned state encodes much of that ordering despite observing only routed outcomes.

\paragraph{External-tiering check.}
Because the main hierarchy is constructed from LiveCodeBench itself, the alignment above could partly reflect a favourable grouping choice.
We therefore repeat the hierarchy construction on the subset of models for which MMLU measurements are available, using MMLU~\cite{hendrycks2021mmlu} to define the tiers while leaving the routing and evaluation benchmark unchanged.
The only change is how the models are grouped before the LiveCodeBench routing run.

Table~\ref{tab:llm_external_tiering} compares that externally constructed hierarchy with same-benchmark tiering on the identical model subset.
The MMLU-tiered hierarchy remains strongly aligned with the LiveCodeBench pass-rate ordering: mean per-seed Kendall $\tau=0.656\pm0.142$ and seed-averaged $\tau=0.930$, compared with $0.637\pm0.192$ and $0.873$ under same-benchmark tiering on the same subset.

\begin{table}[t]
\centering
\caption{External-benchmark tiering check. Both rows are routed and evaluated on the same LiveCodeBench subset, while only the hierarchy construction differs.}
\label{tab:llm_external_tiering}
\begin{tabular}{lcc}
\toprule
Tier construction & Mean per-seed Kendall $\tau$ & Seed-averaged $\tau$ \\
\midrule
LiveCodeBench tiering & $0.637\pm0.192$ & $0.873$ \\
MMLU tiering          & $0.656\pm0.142$ & $0.930$ \\
\bottomrule
\end{tabular}
\end{table}

The external-tiering result is deliberately narrow.
BOHM is an attribution of a particular hierarchy, so invariance to hierarchy construction is neither expected nor desirable.
The check establishes only that the strong state--quality structure in the LLM study is not reducible to defining the hierarchy from the same benchmark used for the diagnostic.
It also foreshadows a broader limitation developed later: hierarchy design is part of the explanatory object, and a grouping that is useful in one domain need not remain useful in another.

\subsection{An externally defined multi-resolution hierarchy}
\label{sec:census}

The LLM study starts from a hierarchy deliberately organised around model quality.
The Census study removes that design freedom.
We use the US Census Bureau's four-level geographic hierarchy: Region, Division, State, and PUMA. It predates this analysis and gives each level an independent institutional meaning.
The question is therefore not whether BOHM can be given a favourable grouping, but what structure an adaptive state develops when it is constrained to an externally specified hierarchy.

The experiment uses the 2022 American Community Survey Public Use Microdata Sample~\cite{acs2022pums}.
We retain adults aged 25--64 with complete income-to-poverty ratio (POVPIP) records and compute mean POVPIP for each PUMA.
After the frozen filtering rules, the hierarchy contains 475 PUMAs, 51 states, 9 divisions, and 4 regions with variable branching.
PUMA quality is rank-normalised to Bernoulli probabilities in $[0.05,0.95]$ and used only to generate the binary outcomes observed by the routing substrate.
We run 50,000 rounds over 20 seeds.

The final routing state is then read at all four Census cuts.
Table~\ref{tab:census} reports both per-seed and seed-averaged alignment between BOHM mass and mean descendant quality at the corresponding level.
The values differ substantially across the hierarchy.
Division has the highest seed-averaged alignment at $\tau=0.722$, PUMA reaches $0.686$ despite containing 475 nodes, and State reaches $0.533$.
Region contains only four nodes, making its rank statistic unstable and unsuitable for a strong significance claim.

\begin{table}[t]
\centering
\caption{State--quality alignment at four simultaneously available cuts of the Census Region$\rightarrow$Division$\rightarrow$State$\rightarrow$PUMA hierarchy. All readouts come from the same final routing state after 50,000 rounds.}
\label{tab:census}
\begin{tabular}{lrrr}
\toprule
Level & Nodes & Per-seed Kendall $\tau$ & Seed-averaged $\tau$ \\
\midrule
Region (descriptive only) & 4   & $0.283\pm0.398$ & $0.333$\textsuperscript{$\ddagger$} \\
Division & 9   & $0.417\pm0.177$ & $0.722$\textsuperscript{$\ast\ast$} \\
State    & 51  & $0.323\pm0.063$ & $0.533$\textsuperscript{$\ast\ast\ast$} \\
PUMA     & 475 & $0.351\pm0.039$ & $0.686$\textsuperscript{$\ast\ast\ast$} \\
\bottomrule
\end{tabular}

\smallskip
{\footnotesize $\ddagger$\,Only four Region nodes. The rank statistic is too coarse for a reliable significance interpretation.
$\ast\ast$\,${p=0.006}$.
$\ast\ast\ast$\,${p<10^{-6}}$.}
\end{table}

\begin{figure}[t]
\centering
\includegraphics[width=0.92\linewidth]{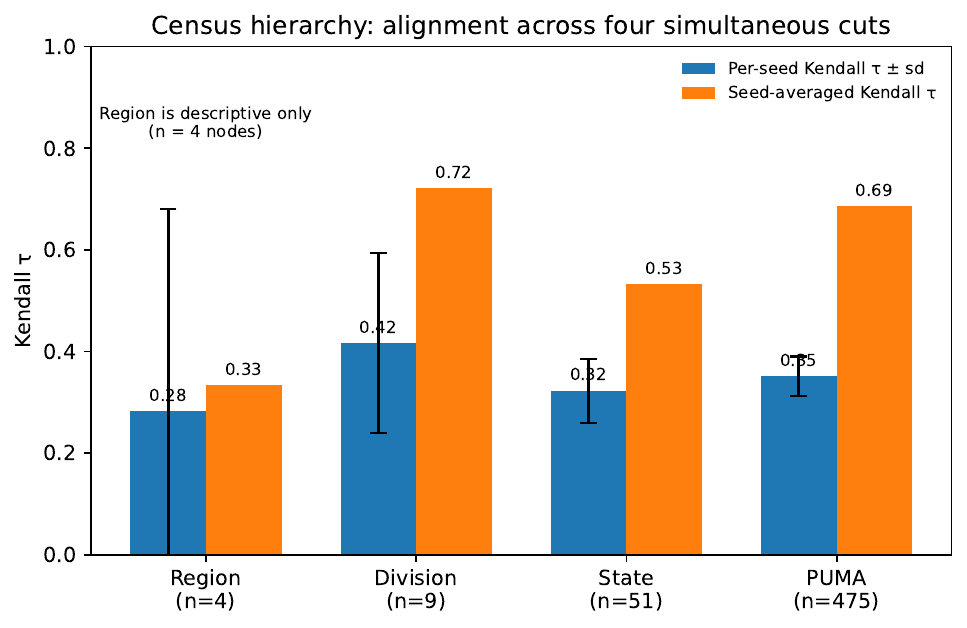}
\caption{Census hierarchy: state--quality alignment across four simultaneous cuts of one learned routing state. Division is the clearest readout, while PUMA also remains substantial despite its much finer granularity.}
\label{fig:census_multiresolution}
\end{figure}

Two features of this result matter for the general argument.

First, the hierarchy is not being reconstructed separately at each level.
Region, Division, State, and PUMA attribution are four cuts through one learned measure.
By Proposition~\ref{prop:refinement}, every State mass is exactly the sum of the PUMA masses below it, every Division mass is the sum of its States, and so on.
The experiment therefore gives an empirical example of a genuinely multi-level explanation rather than four independently computed attribution vectors.

Second, alignment is not monotone in depth.
PUMA alignment is stronger than State alignment even though PUMA is the finest and largest cut.
Conversely, Division is stronger than both.
This is useful evidence against interpreting ``fine resolution'' as inherently weak or ``coarse resolution'' as inherently reliable.
What appears clearly in the routing state depends on the signal structure and on the history accumulated within the given hierarchy.
That point becomes operational in the telecom study, where the Cell cut is usually much less aligned than Region or Site.

\subsection{Resolution in production cellular hierarchies}
\label{sec:telecom}

The Census study shows that one learned state can carry different amounts of externally recognisable structure at different cuts.
We next ask whether that distinction matters in a deployed engineering hierarchy rather than an institutional taxonomy.

\paragraph{Campaign.}
Four telecom networks are represented as operator-local Region$\rightarrow$Site$\rightarrow$Cell trees.
The frozen trees differ substantially in size (Table~\ref{tab:telecom_topology}), which is useful because a resolution effect that appears only in one hierarchy could otherwise be an artefact of one construction.

\begin{table}[t]
\centering
\caption{Frozen telecom hierarchy sizes used in the full local-exposure campaign.}
\label{tab:telecom_topology}
\begin{tabular}{lrrr}
\toprule
Network & Regions & Sites & Cells \\
\midrule
Network A & 55 & 3,517 & 18,547 \\
Network B     & 26 & 2,103 & 6,203 \\
Network C     & 24 & 2,226 & 12,594 \\
Network D      & 59 & 3,014 & 8,104 \\
\bottomrule
\end{tabular}
\end{table}

The candidate set contains 46 KPIs.
A network--KPI pair enters the campaign when the KPI is populated on at least 50\% of eligible Cells, producing 177 population-gated pairs and 1,770 routing trajectories over ten seeds.
Each trajectory runs for 3.0 million local rounds with $\eta=0.02$ and $\epsilon=0.05$.
The cross-pair comparison uses the preselected 2.5M local-round checkpoint. The 3.0M states are retained as a sensitivity endpoint.
All 1,770 trajectories pass the frozen integrity checks.

Ten gated pairs have constant KPI values and therefore no meaningful three-level rank comparison.
The main resolution analysis uses the remaining 167 pairs for which Region, Site, and Cell correlations are all defined.
For each pair, Cell quality is an operator-local tie-aware rank transform.
Site and Region diagnostics use mean descendant transformed quality.
The quantity compared across cuts is therefore \emph{state--quality alignment}: Kendall $\tau_b$ between the BOHM mass induced by the learned routing state and the external node-quality ordering at the same cut.
It is a diagnostic of the ordering encoded by the learned state at that cut.

\paragraph{Where is learned structure most clearly expressed?}
Table~\ref{tab:telecom_depth} gives the headline result.
Site has the highest state--quality alignment in 101 of 167 fully defined comparisons, Region in 65, and Cell in one.
The pattern is not a universal preference for Site.
The network-level behaviour is heterogeneous: Network A is strongly Site-oriented, Network B is mixed, Network C is close to a Region/Site split, and Network D is Region-oriented.
Only a small minority of KPIs choose the same resolution on all operators in the archived campaign analysis.
The hierarchy and deployment therefore matter alongside the KPI itself.

\begin{table}[t]
\centering
\caption{Resolution with highest state--quality alignment in the 167 telecom-network comparisons for which Region, Site, and Cell correlations are all defined.}
\label{tab:telecom_depth}
\begin{tabular}{lrr}
\toprule
Resolution & Cases & Fraction \\
\midrule
Region & 65 & 38.9\% \\
Site   & 101 & 60.5\% \\
Cell   & 1 & 0.6\% \\
\bottomrule
\end{tabular}
\end{table}

\begin{figure}[t]
\centering
\includegraphics[width=0.88\linewidth]{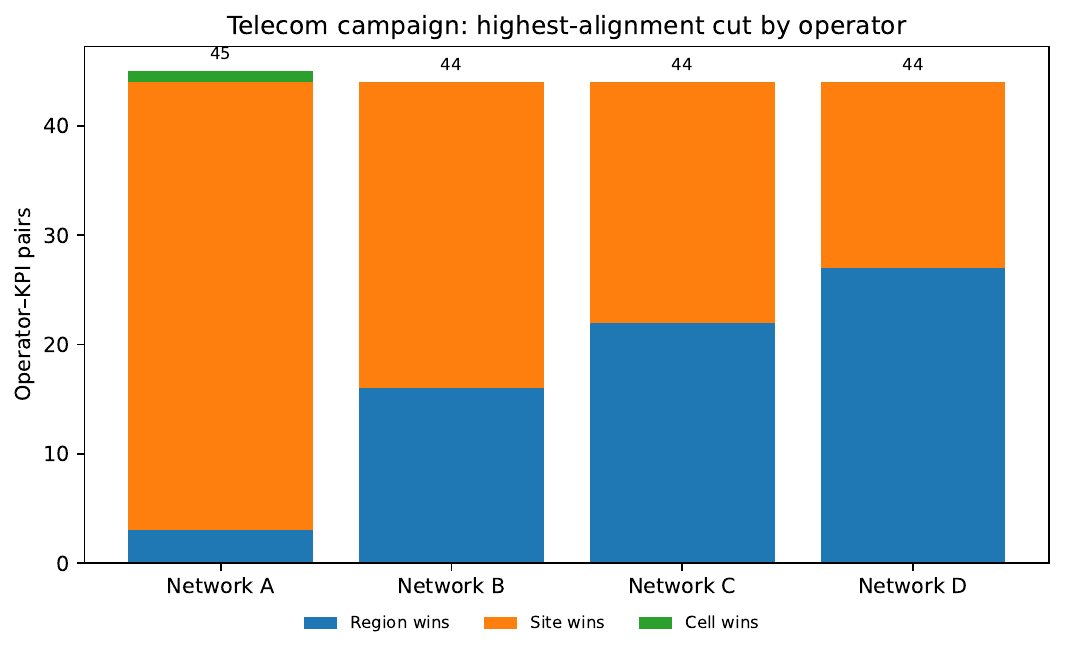}
\caption{Telecom campaign: the cut with highest state--quality alignment varies materially across operators. Network A is strongly Site-oriented, Network B is mixed, Network C is near a Region/Site split, and Network D is Region-oriented.}
\label{fig:telecom_resolution_by_operator}
\end{figure}

The absence of Cell from almost all highest-alignment cases should not be read as a theorem that Cell-level state is uninformative.
Two follow-up analyses ask whether the observed scale preferences track actual hierarchy structure.

\paragraph{Destroying Site structure.}
Four high-Site-alignment cases were selected before perturbation analysis.
Within each Region, Cell quality is permuted so that the Region partition is preserved while the original Site-within-Region organisation is destroyed.
Let
\[
D_{\mathrm{Site}}
=
(\tau_S-\tau_R)_{\mathrm{original}}
-
(\tau_S-\tau_R)_{\mathrm{shuffled}} .
\]
The archived reductions are $+0.410$, $+0.444$, $+0.215$, and $+0.312$.
The Site-over-Region alignment gap therefore decreases in all four cases.
Site ceases to be the highest-alignment level in three.
In the Network B case, Site remains highest after shuffling with a residual Site-over-Region gap of $+0.183461$.

These perturbations are intentionally reported descriptively.
The original condition uses correlation of the ten-seed-averaged attribution, whereas the shuffled condition was summarised as mean per-run correlation across the perturbation trajectories.
The archived test therefore does not provide a matched estimator for the two conditions and does not propagate uncertainty in the original baseline or clustering by permutation.
The safe conclusion is that Site-level alignment is sensitive to Site-organised structure in all four selected cases, not that a common causal effect has already been estimated.

The corresponding Region-side perturbations are more heterogeneous.
One selected Network C case behaves like a clean Region-structured example.
Other Region-dominant cases do not reduce to a single Region component.
This failure of a simple ``one dominant scale determines the readout'' rule motivated a direct component decomposition.

\paragraph{Exact Region/Site/Cell decomposition.}
We select two Region-dominant Network D KPIs for an exact decomposition: \texttt{ActiveUe\_UL\_mean} and \texttt{ActiveUe\_DL\_mean}. For each KPI, the real Cell quality field is decomposed as
\begin{equation}
q_{rsc}
=
\mu + R_r + S_{rs} + C_{rsc},
\label{eq:telecom_decomp}
\end{equation}
where $R$ is the Region mean effect, $S$ is the Site-within-Region effect, and $C$ is the residual Cell-within-Site effect.
Seven deterministic fields are constructed from the combinations
$F$, $RS$, $RC$, $SC$, $R$, $S$, and $C$.
A common contraction around $\mu$ is applied per KPI when needed to keep the Bernoulli probabilities in range.
Topology, routing parameters, and seeds remain fixed.
All 140 trajectories pass the independent state checks.

The decomposition rejects the simplest variance explanation.
For UL, the Region/Site/Cell variance shares are
\[
0.133/0.564/0.303.
\]
for DL they are
\[
0.105/0.575/0.320.
\]
Thus Site contains the majority of raw variance in both KPIs.
Nevertheless, the Full decomposition reruns are Region-dominant: Region alignment is $0.5815$ for UL and $0.4775$ for DL.
Those Full values belong to the contracted decomposition experiment and are not the unmodified full-campaign trajectories.

The isolated fields clarify what the result does and does not mean (Table~\ref{tab:telecom_isolated}).
Region-only variation produces strong Region alignment.
Site-only variation produces positive Site alignment, with Region correlation undefined because there is no Region-level quality variation.
Cell-only variation produces positive Cell alignment, while both Region and Site correlations are undefined.
Fine-scale structure is therefore visible when it exists, but the absolute Cell alignment is small.

\begin{table}[t]
\centering
\caption{Scale-isolated Network D controls. N/A means the external quality is constant at that level, so a ranking correlation is undefined.}
\label{tab:telecom_isolated}
\begin{tabular}{llccc}
\toprule
KPI & Field & Region $\tau$ & Site $\tau$ & Cell $\tau$ \\
\midrule
ActiveUe UL & $R$ only & 0.4717 & 0.0705 & 0.0724 \\
             & $S$ only & N/A    & 0.2160 & 0.2093 \\
             & $C$ only & N/A    & N/A    & 0.0251 \\
\midrule
ActiveUe DL & $R$ only & 0.6482 & 0.1463 & 0.1604 \\
             & $S$ only & N/A    & 0.1946 & 0.1916 \\
             & $C$ only & N/A    & N/A    & 0.0208 \\
\bottomrule
\end{tabular}
\end{table}

The full seven-condition results are reported in Appendix~\ref{app:telecom_protocol}.
They do not support the earlier idea that Region alignment is generically ``propped up'' by nested Site structure.
For UL, removing Site from the Full field changes Region alignment only modestly.
For DL, Region alignment actually increases when Site is removed.
The more defensible conclusion is narrower and more useful: raw variance share does not determine the cut at which the learned routing state is most clearly ordered.

Taken together, the campaign, perturbation, and decomposition results establish that hierarchy levels are not merely labels attached to one leaf ranking.
The same learned state can contain materially different structure at Region, Site, and Cell cuts.
In these networks the most clearly ordered readout is usually mesoscopic, but the scale is operator- and KPI-dependent rather than fixed by BOHM.

\subsection{Same benchmark, different attribution objects}
\label{sec:shapley_cached}

The LiveCodeBench experiment also permits an unusually clean comparison with Shapley attribution because the complete $880\times18$ correctness matrix is available.
The complete matrix lets both attribution objects be compared with the same empirical pass-rate ordering while retaining their different semantics.

For problem $q$, let $A_q$ be the set of models that solve it.
We define the coalition value
\[
v_q(S)
=
\begin{cases}
1, & S\cap A_q\neq\emptyset,\\
0, & S\cap A_q=\emptyset.
\end{cases}
\]
This is an OR game: a coalition succeeds if it contains at least one solver.
The exact Shapley value is available in closed form.
On a solved problem, all members of $A_q$ are symmetric and share the unit coalition value equally, so model $i$ receives
\begin{equation}
\phi_{i,q}
=
\begin{cases}
1/|A_q|, & i\in A_q,\\
0, & i\notin A_q.
\end{cases}
\label{eq:or_shapley}
\end{equation}
A model's final Shapley value is the mean of Eq.~\ref{eq:or_shapley} over the 880 problems.
No permutation sampling is needed.

Before comparing numbers, Table~\ref{tab:bohm_shapley_semantics} states the distinction that matters for this paper.
The two methods consume different information and answer different questions.

\begin{table}[t]
\centering
\caption{Routing-state attribution and Shapley attribution are different explanatory objects.}
\label{tab:bohm_shapley_semantics}
\begin{tabular}{p{0.20\linewidth}p{0.34\linewidth}p{0.34\linewidth}}
\toprule
 & BOHM & Shapley attribution \\
\midrule
Object & Persistent routing-state mass & Coalition value over component subsets \\
Question & Where has the deployed hierarchy learned to place preference? & What marginal contribution does a component make across counterfactual coalitions? \\
Requires & Local simplex state from realised routing history & Exact or estimated values for alternative coalitions \\
Hierarchy & Native to the deployed routing tree, with exact consistency across cuts & Determined by the chosen player and coalition game \\
Disagreement means & Deployed preference differs from the counterfactual contribution view & Same difference, expressed from the coalition side \\
\bottomrule
\end{tabular}
\end{table}

Table~\ref{tab:llm_shapley} then reports one numerical comparison on the cached LiveCodeBench matrix.
The exact Shapley ranking has Kendall alignment $0.980$ with empirical pass rate.
The seed-averaged routing state has alignment $0.928$.
Those values characterize two different objects.
The Shapley value is computed from the complete counterfactual game defined by the full correctness matrix.
BOHM reads the preference state produced by selective routing.

\begin{table}[t]
\centering
\caption{The same LiveCodeBench system viewed through routing-state attribution and an exact OR-game Shapley value. The Kendall column reports alignment of each object with empirical model pass rate.}
\label{tab:llm_shapley}
\begin{tabular}{lrrr}
\toprule
Object & Kendall $\tau$ & Routed invocations & Distinct matrix cells required \\
\midrule
BOHM, one routed state & $0.739$\textsuperscript{$\dagger$} & 880 & $\leq880$ \\
BOHM, mean of 20 states & 0.928 & 17,600 & 9,323 \\
Exact Shapley & 0.980 & 15,840 & 15,840 \\
\bottomrule
\end{tabular}

\smallskip
{\footnotesize $\dagger$ Mean over individual routed states is $0.739\pm0.079$.}
\end{table}

The information-access difference is more important than the raw counts.
A single deployed routing state can exist after observing one selected model on each problem.
The exact Shapley object used here exists only because every model has already been evaluated on every problem.
Once that matrix exists, exact Shapley is straightforward and the comparison should not be framed as a computational victory for BOHM.

The semantic difference is also visible in what each object would mean if the two rankings disagreed.
BOHM would still be reporting the selection state actually learned by the router.
The Shapley value would still be reporting the component's average marginal contribution to the OR game.
A disagreement would therefore identify a difference between deployed preference and counterfactual coalition value, not an error in one of the attribution definitions.

\subsection{Agentic routing: deployed preference versus counterfactual contribution}
\label{sec:agentic}

The semantic distinction becomes more consequential when the coalition game is not already present as a complete matrix.
We instrument an agentic harness in which one of five driver models selects among five tools.
The drivers are DeepSeek-V3.2, GLM-5.1-FP8, Qwen3.6-35B-A3B-FP8, Qwen2.5-32B-Instruct, and Devstral-Small-2-24B.
The tools are Qwen3-Coder-480B-A35B-Instruct-FP8, gpt-oss-120b, DeepSeek-V3.2, Qwen3-32B, and Qwen2.5-14B-Instruct-1M.

We use seven benchmarks spanning code and knowledge tasks:
CodeContests~\cite{li2022codecontests},
LiveCodeBench~\cite{jain2024livecodebench},
MBPP~\cite{austin2021mbpp},
BigCodeBench~\cite{zhuo2024bigcodebench},
EvalPlus~\cite{liu2023evalplus},
MMLU~\cite{hendrycks2021mmlu}, and
MATH~\cite{hendrycks2021math}.
For each of the 35 driver--benchmark cells, the deployed trace contains 100 routed problems.

To construct a counterfactual game, we separately enumerate all $2^5-1=31$ non-empty tool menus.
For every subset $S$, the driver is reprompted with only the tools in $S$ available, chooses one of them, and the resulting answer is graded.
The complete coalition lattice is therefore measured rather than sampled.
The deployed trace is replayed through a non-uniform $[3,2]$ BOHM hierarchy that groups three mixture-of-experts tools and two dense tools.

\paragraph{Deployment explores only a small part of the coalition space.}
Routing is strongly concentrated.
Across the 35 cells, the median share assigned to the most-selected tool is $0.65$, with a range from $0.39$ to $1.00$.
Thirty of 35 cells route at least half of their requests to one tool.
For example, GLM-5.1-FP8 assigns 69\% of its LiveCodeBench routes to DeepSeek-V3.2, while Qwen2.5-32B-Instruct sends every BigCodeBench request to Qwen3-Coder-480B.
The counterfactual coalition experiment therefore asks about many menus the deployed system did not actually traverse.

\paragraph{Agreement depends on how well deployed preference matches observed tool quality.}
Across the 35 cells, Kendall agreement between BOHM and Shapley tool rankings ranges from $-0.80$ to $+1.00$.
In the nine cells where the driver's deployed top pick is also the empirically best tool on that benchmark, mean agreement is $+0.22$.
In the remaining 26 cells it is approximately $+0.01$.
The difference, $+0.21$, is descriptive rather than a universal predictive law.
It is nevertheless in the direction expected from the semantics of the two objects: agreement is higher when the deployed preference already concentrates on the tool that performs best in the observed cell.

A cross-family check gives the same qualitative result.
Twenty-six of the 35 driver/top-tool pairs cross model families.
Within those cells, the corresponding agreement gap is $+0.34$, so the all-cell pattern is not explained by Qwen-family drivers simply selecting Qwen-family tools.

\paragraph{Worked example.}
Table~\ref{tab:agentic_worked} shows two LiveCodeBench cells.
Qwen3.6-A3B routes most often to gpt-oss-120b, which is also its empirically strongest tool.
Its BOHM and Shapley top tools therefore agree.
GLM-5.1-FP8 instead routes 69\% of requests to DeepSeek-V3.2 even though gpt-oss-120b performs better on the same 100 problems.
Its BOHM top tool is DeepSeek-V3.2, reflecting the deployed state, while its Shapley top is gpt-oss-120b, reflecting the counterfactual game.

\begin{table}[t]
\centering
\caption{Two LiveCodeBench agentic cells. BOHM reflects deployed routing preference, while Shapley reflects the complete restricted-menu coalition game.}
\label{tab:agentic_worked}
\begin{tabular}{lrrrr}
\toprule
& \multicolumn{2}{c}{Qwen3.6-A3B driver} & \multicolumn{2}{c}{GLM-5.1-FP8 driver} \\
Tool & BOHM & SHAP & BOHM & SHAP \\
\midrule
gpt-oss-120b            & \textbf{0.42} & \textbf{+0.34} & 0.30 & \textbf{+0.29} \\
DeepSeek-V3.2           & 0.18 & +0.23 & \textbf{0.36} & +0.26 \\
Qwen3-Coder-480B        & 0.12 & +0.20 & 0.04 & +0.19 \\
Qwen3-32B               & 0.13 & $-0.09$ & 0.15 & $-0.07$ \\
Qwen2.5-14B-Instruct-1M & 0.15 & $-0.03$ & 0.15 & $-0.11$ \\
\bottomrule
\end{tabular}
\end{table}

The example captures the role of counterfactual attribution in this paper.
Shapley can reveal capability that deployment did not exploit.
BOHM can reveal that the deployed hierarchy did not place its preference there.
The difference between them is therefore itself a system diagnostic.

\section{Hierarchy and resolution sensitivity}
\label{sec:robustness}

BOHM exactly reports the hierarchy state it is given.
That makes hierarchy construction part of the explanatory model rather than an implementation detail.
Three analyses quantify this dependence.

\subsection{Natural versus random grouping}

The main LLM hierarchy groups models with similar LiveCodeBench quality.
To test whether the resulting structure survives arbitrary regrouping, we compare that hierarchy with ten random shuffles of the same 18 models across the three tiers, each evaluated over 20 routing seeds.
Table~\ref{tab:natural_random} reports both leaf-level state--quality alignment and the spread of the three tier masses.

\begin{table}[t]
\centering
\caption{Effect of hierarchy grouping in the 18-model LLM study. Random grouping shuffles the same models across tiers.}
\label{tab:natural_random}
\begin{tabular}{lcc}
\toprule
Grouping & Kendall $\tau$ & Tier-mass spread \\
\midrule
Quality-coherent hierarchy & 0.739 & 0.279 \\
Random hierarchy           & 0.507 & 0.142 \\
\bottomrule
\end{tabular}
\end{table}

The quality-coherent hierarchy produces 46\% higher rank alignment and approximately twice the tier-mass separation.
This does not show that BOHM fails on a random hierarchy.
It shows that BOHM faithfully exposes the structure of the hierarchy it is asked to describe.
A hierarchy that cuts across the external quality ordering naturally produces a state whose group masses are less aligned with that ordering.

\subsection{Domain-specific hierarchy}

The same 18 models are also evaluated on five coding benchmarks whose model rankings differ materially:
BigCodeBench, LiveCodeBench, CodeContests, HumanEval, and MBPP.
For each benchmark we compare the fixed LiveCodeBench hierarchy with a hierarchy retiered from that benchmark's own pass-rate ordering.
No routing hyperparameters are changed.

\begin{table}[t]
\centering
\caption{Fixed LiveCodeBench tiering versus domain-specific tiering. $\Delta$ is domain-specific minus fixed state--quality alignment.}
\label{tab:domain_tiering}
\begin{tabular}{lrrr}
\toprule
Benchmark & Fixed $\tau$ & Domain-specific $\tau$ & $\Delta$ \\
\midrule
BigCodeBench & 0.289 & 0.370 & +0.081 \\
LiveCodeBench & 0.739 & 0.715 & $-0.024$ \\
CodeContests & 0.319 & 0.385 & +0.066 \\
HumanEval & 0.105 & 0.476 & +0.371 \\
MBPP & 0.418 & 0.561 & +0.143 \\
\bottomrule
\end{tabular}
\end{table}

\begin{figure}[t]
\centering
\includegraphics[width=0.92\linewidth]{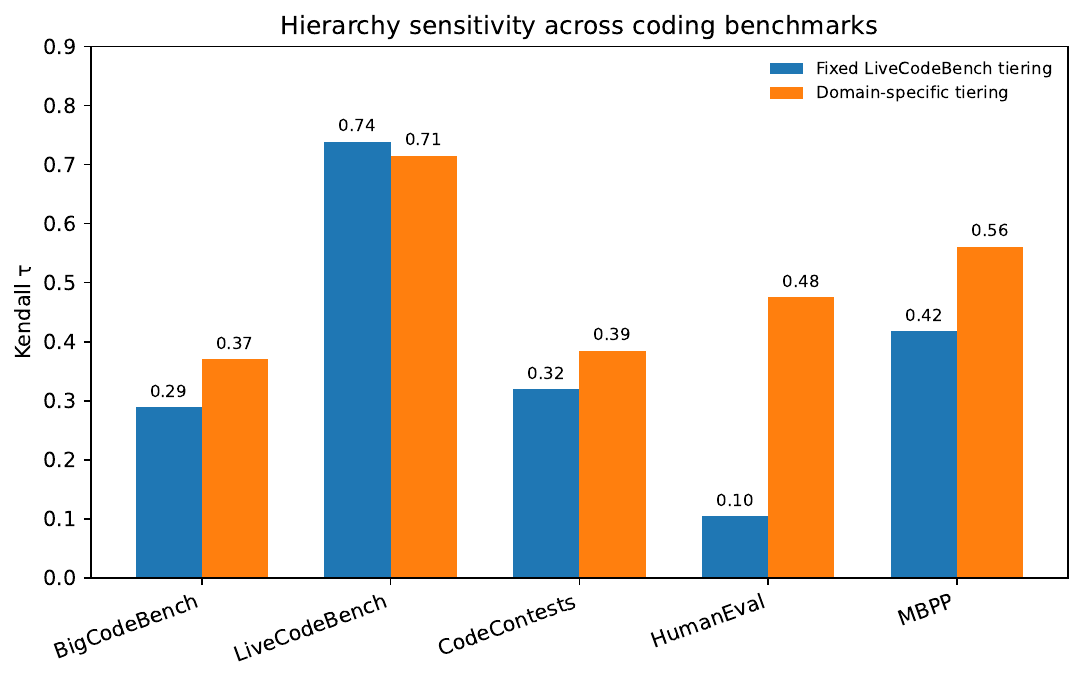}
\caption{Hierarchy sensitivity across coding benchmarks. Re-tiering to the benchmark domain usually improves state--quality alignment, but the size of the improvement varies substantially across domains.}
\label{fig:hierarchy_sensitivity}
\end{figure}

HumanEval is the clearest case.
Reusing the LiveCodeBench hierarchy leaves very little quality-aligned structure in the learned state, whereas retiering for HumanEval increases alignment from $0.105$ to $0.476$.
The hierarchy is therefore not a neutral container around a fixed attribution vector.
It is part of the state whose explanation BOHM exposes.

\subsection{Depth}

Finally, we isolate depth on balanced synthetic trees with branching factor three.
Depths one through four contain 3, 9, 27, and 81 leaves.
The number of routing rounds is increased with depth to 5,000, 20,000, 60,000, and 120,000.
Over ten seeds, final state--quality alignment is approximately $1.00$, $0.71$, $0.72$, and $0.67$ respectively.

The important observation is not a monotone loss of alignment with depth.
Depths two and three are almost identical, and depth four remains substantial.
What does change is the amount of routing history needed before lower-level states stabilise.
That behaviour is consistent with the local-support argument in Section~\ref{sec:evidence}: a deeper hierarchy introduces more local states whose weights must be learned from the finite routed history.

\section{Discussion}
\label{sec:discussion}

\subsection{Routing state is a first-class explanatory object}

The main contribution is not a new estimator of component quality.
It is the observation that persistent hierarchical routing state already defines a coherent attribution measure.
Once the local weights exist, the node masses in Eq.~\ref{eq:node_mass} are exact.
No external target is required to make BOHM valid.

This distinction resolves an ambiguity common in attribution studies.
A method can faithfully describe a system state even when that state is poorly aligned with an external notion of component merit.
For adaptive routing, that difference is often the interesting part.
A high BOHM mass on a weak component is evidence about the router, not necessarily about the component.

\subsection{The hierarchy is part of the explanation}

A flat attribution collapses several possible explanatory scales into one.
BOHM retains the hierarchy and therefore preserves the distinction between group-level and leaf-level preference.
The telecom results show why that matters.
The learned state is usually more coherent at Region or Site than at Cell, while scale-isolated controls show that Cell-level structure can still be present.

The practical implication is not that Site is the ``correct'' resolution.
It is that explanatory resolution should be treated as a property of the state and evidence, not fixed in advance.
A coarse readout may hide stable within-group structure.
A fine readout may expose distinctions that the routing history has not supported strongly enough to order.

\subsection{Finer readouts depend on thinner local routing histories}

Proposition~\ref{prop:refinement} says that changing the readout resolution loses no BOHM mass: every coarse mass is recovered exactly by summing the finer descendants.
What changes with depth is the routing state from which those masses are formed.
Equation~\ref{eq:evidence_conservation} records how many routed observations reached the local states on a cut, while Eq.~\ref{eq:path_error} shows that a deeper mass contains more locally estimated factors.

The practical consequence is specific to stateful routing.
A fine readout may be perfectly well defined even when some of the local weights it contains have been updated from relatively little experience.
That does not make the BOHM calculation less correct.
It means that the underlying learned state may contain less clearly ordered structure at that cut.
How much local history is enough depends on the routing update, quality gaps, and realised exposure, so we do not claim a universal optimal-resolution law here.

\subsection{Counterfactual attribution is complementary}

Shapley attribution and BOHM should not be forced into a winner--loser comparison.
They are useful precisely because they answer different questions.
A counterfactual method can reveal unrealised capability that the deployed router never explored.
BOHM can reveal that the deployed system failed to place preference on that capability.
The gap between them is therefore potentially diagnostic.

This complementarity also clarifies information requirements.
BOHM can be computed whenever the persistent hierarchical state exists.
A Shapley value requires a coalition value function, which may be directly available, approximable, or obtainable only through additional interventions.
Those are deployment facts rather than universal computational advantages of either method.

\subsection{Scope}

BOHM requires persistent local routing state.
It is not automatically defined for a stateless prompt router or a token-conditional MoE gate unless such outputs are accumulated into an appropriate state.
The routing substrate used here is stationary and outcome-driven. Context-dependent and non-stationary states require corresponding extensions.

Hierarchy design also matters.
BOHM faithfully describes the hierarchy it is given.
If the hierarchy groups together components whose roles change sharply with context, the resulting state will reflect that choice.
Additional sensitivity experiments in Appendix~\ref{app:sensitivity} show both depth effects and domain-specific hierarchy effects.

Finally, BOHM is descriptive.
It does not by itself establish causal contribution, normative merit, or optimal routing.
Those questions require external objectives or interventions.

\section{Conclusion}

A stateful adaptive hierarchy already contains an explanation of its own behaviour.
BOHM makes that explanation explicit by turning local routing weights into a probability measure over the entire tree.
The resulting attribution is multi-resolution and exactly consistent across cuts, while external diagnostics can be added when the learned state is to be compared with an independent performance ordering.

The empirical studies show why retaining the hierarchy matters.
Learned structure can appear at several resolutions. In telecom networks it is usually most clearly expressed above the individual Cell level. Comparison with counterfactual attribution reveals a distinct view of component capability from the one encoded by deployed preference.

The central consequence is that attribution resolution is not merely a visualisation choice.
A finer readout preserves the coarser BOHM masses exactly, but the state being read was learned from finite evidence distributed through the hierarchy.
Reading the same routing state at several hierarchy levels therefore shows both what the system has learned and where that learned structure is clearly resolved.

\section{Reproducibility and provenance}
\label{sec:reproducibility}

The empirical results in this paper come from several experiment families with different state-generation regimes.
Table~\ref{tab:repro_map} records the primary routing horizons and seed structure so that the estimands are not conflated across sections.
The public research package is organised around machine-readable manifests rather than the manuscript tables alone.

\begin{table}[t]
\centering
\caption{Primary experiment families and routing-state generation. ``Seeds'' refers to independent routing-state initialisations unless otherwise noted.}
\label{tab:repro_map}
\begin{tabular}{p{0.23\linewidth}p{0.17\linewidth}p{0.17\linewidth}p{0.32\linewidth}}
\toprule
Family & Horizon & Seeds & Primary retained object \\
\midrule
18-LLM LiveCodeBench & 880 problems & 20 & final leaf/tier state; exact cached correctness matrix \\
Census & 50,000 rounds & 20 & Region/Division/State/PUMA masses \\
Telecom full campaign & 3.0M local rounds & 10 fixed primes & checkpoints + leaf mass for 177 operator--KPI pairs \\
Network D decomposition & 2.5M local rounds & 10 fixed primes & seven component fields for two KPIs \\
Agentic & 100 deployed problems/cell & 20 first-prime routing seeds & deployed trace + all 31 non-empty counterfactual menus \\
Depth sensitivity & 5k--120k rounds & 10 & synthetic final state by depth \\
\bottomrule
\end{tabular}
\end{table}

\paragraph{Integrity checks.}
The telecom full campaign contains 1,770 completed trajectories and passes the frozen state-integrity checks for every trajectory.
The Network D component-decomposition study contains 140 trajectories, all of which pass independently reconstructed checks on finite positive weights, simplex conservation, identity edges, endpoint counts, attribution reconstruction, and frozen tree/quality/state hashes.
The numerical replay gate in Appendix~\ref{app:substrate} separately verifies that the retained LLM, Census, and agentic headline results are unchanged by the stable finite-precision update.

\paragraph{Analysis estimands.}
Per-seed rank correlation and rank correlation of seed-averaged attribution are distinct estimands and are reported separately where both are used.
The telecom cross-pair headline uses Kendall $\tau_b$ of seed-averaged state at the preselected 2.5M checkpoint under the operator-local quality transform.
The agentic BOHM--Shapley comparison uses tool rankings within each fixed driver--benchmark cell.
Counterfactual coalition values in that study are measured from the complete restricted-menu intervention table rather than inferred from deployed-route frequency.

\paragraph{Release structure.}
Table~\ref{tab:release_map} states the intended release boundary explicitly.
Raw operator measurements are not redistributed where confidentiality or licensing prevents release, but the derived artefacts used to audit the reported aggregate telecom results are included.

\begin{table}[t]
\centering
\caption{Planned public research-package contents.}
\label{tab:release_map}
\begin{tabular}{p{0.25\linewidth}p{0.65\linewidth}}
\toprule
Artefact class & Planned release \\
\midrule
Code and manifests & Source code, frozen experiment manifests, seed lists, analysis scripts \\
LLM/Census/agentic inputs & Precomputed LLM outcome matrices where licensing permits, Census-derived hierarchy artefacts, deployed agentic traces, complete restricted-menu intervention tables \\
Telecom derived artefacts & Frozen hierarchy metadata, protocol manifests, aggregate attribution arrays, integrity summaries, machine-readable analysis outputs; raw operator measurements excluded where restricted \\
Numerical audit & Failure capture, stress-test scripts and reports, old-versus-corrected replay summaries \\
\bottomrule
\end{tabular}
\end{table}

\section*{Data and Code Availability}
Code and reproducibility artefacts will be released with the public research package subject to underlying licences and operator-data restrictions.
The release is intended to reproduce every retained table and statistic in the manuscript from the frozen derived inputs.
Where raw operator data cannot be distributed, the package will state that boundary explicitly and provide the derived artefacts needed to audit the reported aggregate results.

\appendix

\section{Adaptive routing substrate and numerical implementation}
\label{app:substrate}

The experiments require a stateful mechanism that updates local routing weights from observed outcomes.
This substrate is not part of the BOHM definition. It is one way of generating the persistent state that BOHM reads.
The experiments use the following proportional redistribution update.

\begin{algorithm}[h]
\caption{Adaptive routing substrate: one round}
\label{alg:substrate}
\begin{algorithmic}[1]
\REQUIRE Tree $\mathcal{T}$ with local weights $\mathbf{w}_v$, learning rate $\eta$, exploration rate $\epsilon$
\STATE Route from root to leaf.
\FOR{each active selector $v$}
  \STATE With probability $\epsilon$, select a child uniformly.
  \STATE Otherwise select from $\mathrm{Categorical}(\mathbf{w}_v)$.
\ENDFOR
\STATE Observe binary outcome $o$ at the selected leaf.
\STATE Update the root selector from $o$.
\STATE Propagate the sign of the parent-weight change down the selected path.
\STATE For a positive signal on child $i$:
\STATE \hspace{1em}$w_i\gets(1-\eta)w_i+\eta$ and $w_j\gets(1-\eta)w_j$ for $j\ne i$.
\STATE For a negative signal on child $i$:
\STATE \hspace{1em}$x\gets w_i$, $r\gets\eta x$, $S\gets\sum_{j\ne i}w_j$.
\STATE \hspace{1em}$w_i\gets x-r$ and $w_j\gets(S+r)(w_j/S)$ for $j\ne i$.
\end{algorithmic}
\end{algorithm}

Adaptive selector nodes have at least two children.
Unary internal nodes are deterministic identity edges with fixed weight one, deterministic pass-through, no exploration draw, and no adaptive update.
They may still record activation counts for provenance.

\subsection{Finite-precision failure mode}

In exact arithmetic, the negative update can be written by inferring sibling mass as $1-w_i$.
That form is unsafe near saturation.
When $w_i$ approaches one, the subtraction can lose almost all relative precision while the true sibling weights remain positive but extremely small.
The implementation therefore computes
\[
S=\sum_{j\ne i}w_j
\]
directly from the sibling entries, removes $r=\eta w_i$ from the selected child, and redistributes it as
\[
w_i'=w_i-r,
\qquad
w_j'=(S+r)\frac{w_j}{S}.
\]
If $S=0$ for a selector with $b\ge2$, the implementation raises an error rather than inventing sibling proportions after underflow.
No post-hoc renormalisation is applied.

The implementation audit captured an explicit failing trace for the older algebraically equivalent factor form.
In that trace, the true sibling mass was approximately $1.4\times10^{-18}$ while the subtraction $1-w_i$ returned $2.66\times10^{-15}$, an overestimate of roughly three orders of magnitude.
With $\eta=0.02$, the selected child released $r\approx0.02$ but the mis-scaled siblings did not recover the released mass, leaving a post-update simplex sum near $0.98$.
At still longer horizons the same cancellation can become more severe.
The failure is therefore a floating-point saturation problem, not a failure of the exact update rule.

\subsection{Stress audit}

The corrected implementation was tested in 28 deliberately difficult configurations. These cover branching factors $2,3,4,$ and $16$, learning rates $0.01,0.02,$ and $0.05$, Bernoulli qualities near $0.01$, $0.5$, and $0.99$, long success streaks followed by failure, repeated-failure regimes, and alternating signals.
Each configuration runs for up to 16 million updates.
All 28 pass: every weight remains finite and strictly positive, no sibling-denominator failure occurs, and the simplex residual remains at numerical-noise scale.
The worst observed residual is $3.5\times10^{-13}$ in the 16-million-update repeated-failure case.
The formerly catastrophic saturation-then-failure case is approximately $1.7\times10^{-15}$.

\begin{table}[t]
\centering
\caption{Numerical audit and replay gate. Replay rows compare the old and corrected implementations using the original data and seeds for the retained manuscript experiments.}
\label{tab:numerical_audit}
\begin{tabular}{p{0.28\linewidth}p{0.62\linewidth}}
\toprule
Check & Result \\
\midrule
Stress suite & 28/28 pass through 16M updates; worst simplex residual $3.5\times10^{-13}$ \\
Saturation/failure & Formerly catastrophic case has residual $\approx1.7\times10^{-15}$ \\
18-LLM replay & Per-seed and seed-averaged rankings unchanged; seed-averaged Kendall $\tau=0.927652$ \\
Census replay & Multi-resolution results unchanged; Division $\tau=0.7222$, PUMA $\tau=0.6857$ \\
35-cell agentic replay & All 35 BOHM--Shapley cell correlations unchanged; headline alignment split unchanged \\
Unary-node handling & Unary internal nodes are fixed identity edges and cannot leak probability mass \\
\bottomrule
\end{tabular}
\end{table}

\subsection{Replay gate for retained evidence}

The numerical correction was applied to the implementations that generate the retained BOHM evidence, after which the three pre-telecom headline families were replayed using the original data and seeds.
The 18-LLM per-seed and seed-averaged rankings are unchanged at manuscript precision, with seed-averaged Kendall $\tau=0.927652$.
The Census multi-resolution results are unchanged, including Division $\tau=0.7222$ and PUMA $\tau=0.6857$.
All 35 agentic BOHM--Shapley cell correlations are unchanged, and the reported top-pick alignment split is unchanged.
These replays establish that the finite-precision correction hardens the implementation without changing the retained headline evidence.

A separate long-horizon filter-sensitivity study is not part of the present evidence base.
Its PISA Minimal/Zero constructions contained unary adaptive routers under the older implementation, so those rows are not presented as reproducible results here.
The correction and exclusion are deliberately narrower than retroactively renormalising or selectively repairing those historical trajectories.

\section{Experimental protocols}
\label{app:protocols}

\subsection{LLM hierarchy}

The LLM study uses 18 models and 880 LiveCodeBench problems~\cite{jain2024livecodebench}.
Models are grouped into three empirical quality tiers, each split into three subgroups of two models.
The stateful wrapper uses $\eta=0.05$, $\epsilon=0.05$, and 20 random seeds.
Each seed processes every problem once.
Only the selected model outcome is exposed to the routing update.
The external diagnostic is Kendall correlation between final BOHM leaf mass and empirical model pass rate.

An external-tiering sensitivity check rebuilds the hierarchy on the subset of models with MMLU measurements and still obtains strong state--quality alignment on the coding benchmark.
This check is included only to show that the LLM result is not wholly an artefact of constructing the hierarchy from the same benchmark used for the diagnostic.

\subsection{Census hierarchy}

\paragraph{Data.}
The Census study uses the 2022 American Community Survey Public Use Microdata Sample~\cite{acs2022pums}.
The source contains person-level socioeconomic records together with Region, Division, State, and PUMA identifiers.
We retain adults aged 25--64 with complete POVPIP records, leaving approximately 4.8 million person records.
The scalar used to generate leaf outcomes is mean POVPIP per PUMA.
PUMAs with fewer than 50 qualifying records are excluded so that the leaf mean is not dominated by very small samples.

\paragraph{Hierarchy.}
The institutional tree is Region (4) $\rightarrow$ Division (9) $\rightarrow$ State (51) $\rightarrow$ PUMA (475).
Branching is variable.
The geographic taxonomy is defined by the Census Bureau rather than by the experiment.
The implementation removes selector nodes with fewer than two active children and caps very large branching at ten children for routing-state stabilisation. Deterministic single-child paths are not treated as adaptive decisions.

\paragraph{Routing protocol.}
Raw PUMA means are rank-normalised and mapped to Bernoulli outcome probabilities in $[0.05,0.95]$.
Each round routes to one PUMA, draws a binary outcome from that probability, and updates the local state on the selected path.
The experiment uses $\eta=0.05$, $\epsilon=0.05$, 50,000 rounds per seed, and 20 seeds.
No coarser attribution is trained separately.

\paragraph{Diagnostics.}
For a PUMA, external quality is its transformed leaf quality.
For an internal node, external quality is the mean quality of its retained descendants.
Kendall $\tau$ is computed between node mass and that external ordering at Region, Division, State, and PUMA cuts, both per seed and after averaging the final node-mass vectors across seeds.
These correlations describe the ordering encoded by the routing state at each cut relative to the corresponding external quality ordering.

\subsection{Production telecom hierarchy}
\label{app:telecom_protocol}

\paragraph{Trees and KPI gate.}
The four operator-local trees contain 45,448 Cell leaves in total.
Their Region/Site/Cell counts are listed in Table~\ref{tab:telecom_topology}.
The frozen candidate list contains 46 KPI definitions.
A pair is retained when the KPI is defined on at least 50\% of eligible Cells, yielding 177 population-gated pairs.
Ten retained pairs have constant raw quality and are excluded from the fully defined depth-comparison denominator:
\texttt{PDSCH\_TypeB\_share} on all four operators, and
\texttt{SsbBeamSwitch\_per\_ROP} and
\texttt{RrcAccessFail\_BbIntens\_per\_ROP} on Network B, Network C, and Network D.
The remaining 167 pairs have defined Region, Site, and Cell correlations.

Leaves with missing KPI values remain present in routing and generate failure outcomes.
Evaluation quality at Site and Region is the mean transformed quality over finite descendants.
Thirteen candidate KPIs are metadata-flagged as non-monotone or two-sided, so their correlations should be interpreted only relative to the chosen monotone structural ordering rather than as direct operational utility.

\paragraph{Routing protocol.}
Each retained pair is run for 3.0 million local rounds under ten fixed prime-number seeds: 2, 3, 5, 7, 11, 13, 17, 19, 23, and 29.
The routing parameters are $\eta=0.02$ and $\epsilon=0.05$.
Checkpoints are written at 0.25M, 0.75M, 1.5M, 2.5M, and 3.0M local rounds.
The cross-pair analysis uses the preselected 2.5M checkpoint.
All 1,770 trajectories pass the frozen state-integrity checks.

\paragraph{Site-structure perturbation.}
The four archived Site-over-Region gap reductions are $+0.410$, $+0.444$, $+0.215$, and $+0.312$.
Site ceases to be the highest-alignment level in three of the four perturbed conditions.
Network B retains a Site-over-Region gap of $+0.183461$.
Because the original and perturbed summaries use unmatched aggregation estimators and the archived analysis does not propagate original-baseline uncertainty or permutation clustering, the perturbation is treated descriptively in the main text.

\paragraph{Network D decomposition.}
For \texttt{ActiveUe\_UL\_mean}, the common contraction factor is $\lambda=0.7039308363$ and the Region/Site/Cell variance shares are $0.133027/0.563763/0.303210$.
For \texttt{ActiveUe\_DL\_mean}, $\lambda=0.6704583512$ and the shares are $0.105269/0.574626/0.320105$.
The exact seven-condition results are shown in Tables~\ref{tab:netd_ul_full} and~\ref{tab:netd_dl_full}.
N/A denotes a level with no true quality variation.

\begin{table}[t]
\centering
\caption{Network D \texttt{ActiveUe\_UL\_mean} decomposition. Kendall $\tau_b$ of seed-averaged BOHM mass at 2.5M rounds.}
\label{tab:netd_ul_full}
\begin{tabular}{lrrr}
\toprule
Field & Region & Site & Cell \\
\midrule
$F$  & 0.5815 & 0.2195 & 0.1807 \\
$RS$ & 0.4541 & 0.2045 & 0.1979 \\
$RC$ & 0.5605 & 0.1064 & 0.0588 \\
$SC$ & N/A    & 0.1966 & 0.1569 \\
$R$  & 0.4717 & 0.0705 & 0.0724 \\
$S$  & N/A    & 0.2160 & 0.2093 \\
$C$  & N/A    & N/A    & 0.0251 \\
\bottomrule
\end{tabular}
\end{table}

\begin{figure}[t]
\centering
\includegraphics[width=0.94\linewidth]{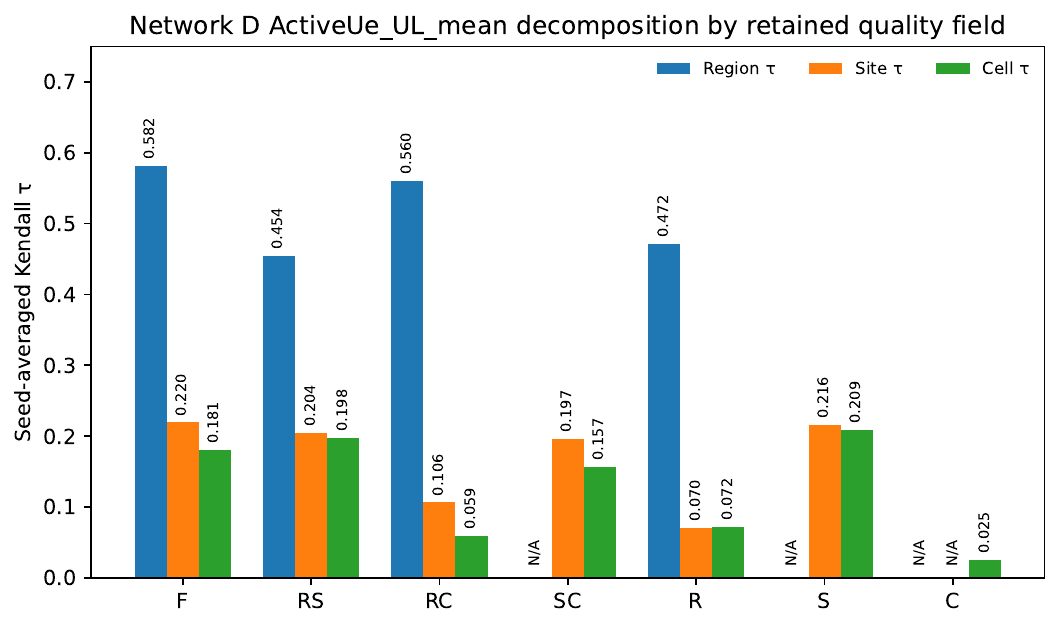}
\caption{Network D \texttt{ActiveUe\_UL\_mean} decomposition visualised. Region dominance persists in the Full and several mixed fields even though Site carries the largest raw variance share.}
\label{fig:netd_ul_decomposition}
\end{figure}

\begin{table}[t]
\centering
\caption{Network D \texttt{ActiveUe\_DL\_mean} decomposition. Kendall $\tau_b$ of seed-averaged BOHM mass at 2.5M rounds.}
\label{tab:netd_dl_full}
\begin{tabular}{lrrr}
\toprule
Field & Region & Site & Cell \\
\midrule
$F$  & 0.4775 & 0.2328 & 0.1959 \\
$RS$ & 0.4752 & 0.1882 & 0.1862 \\
$RC$ & 0.5266 & 0.1295 & 0.0598 \\
$SC$ & N/A    & 0.1888 & 0.1486 \\
$R$  & 0.6482 & 0.1463 & 0.1604 \\
$S$  & N/A    & 0.1946 & 0.1916 \\
$C$  & N/A    & N/A    & 0.0208 \\
\bottomrule
\end{tabular}
\end{table}

\begin{figure}[t]
\centering
\includegraphics[width=0.94\linewidth]{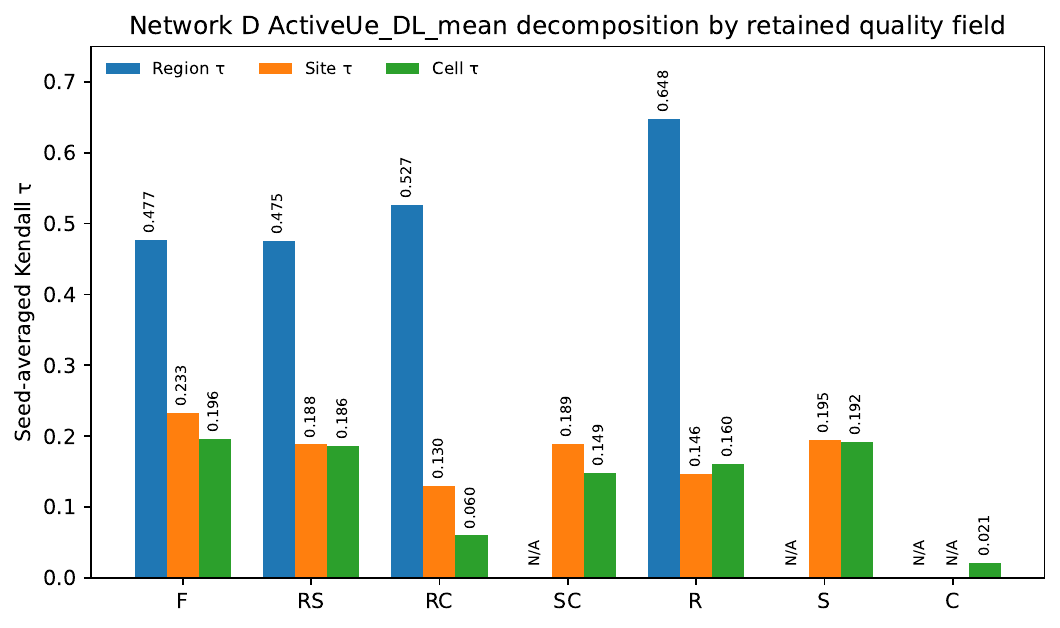}
\caption{Network D \texttt{ActiveUe\_DL\_mean} decomposition visualised. Region alignment increases when Site is removed from the Full field, reinforcing that raw variance share does not determine the most clearly ordered cut.}
\label{fig:netd_dl_decomposition}
\end{figure}

The decomposition is an exact nested mean residualisation:
$R_r$ is the Region mean minus the grand mean,
$S_{rs}$ is the Site mean minus its Region mean,
and $C_{rsc}$ is Cell quality minus its Site mean.
Reconstruction error is at machine precision.
All 140 decomposition trajectories pass independently reconstructed state checks.

\subsection{Agentic orchestrator}

\paragraph{Drivers, tools, and benchmarks.}
The five drivers are DeepSeek-V3.2, GLM-5.1-FP8, Qwen3.6-35B-A3B-FP8, Qwen2.5-32B-Instruct, and Devstral-Small-2-24B.
The five tools are Qwen3-Coder-480B-A35B-Instruct-FP8, gpt-oss-120b, DeepSeek-V3.2, Qwen3-32B, and Qwen2.5-14B-Instruct-1M.
The seven benchmarks are CodeContests, LiveCodeBench, MBPP, BigCodeBench, EvalPlus, MMLU, and MATH.
Each driver--benchmark cell uses 100 problems sampled deterministically from the benchmark.

\paragraph{Deployed trace and coalition trace.}
Each cell contains 100 deployed routing decisions plus tool execution.
For counterfactual attribution we enumerate all 31 non-empty subsets of the five-tool menu.
For subset $S$, the same driver is reprompted with only the tools in $S$ available and must choose one of them.
The chosen tool runs and its answer is graded.
The full coalition lattice is therefore measured rather than sampled.

The deployed trace is replayed through a depth-two non-uniform $[3,2]$ BOHM hierarchy.
The three MoE tools are Qwen3-Coder-480B-A35B-Instruct-FP8, gpt-oss-120b, and DeepSeek-V3.2.
The dense group contains Qwen3-32B and Qwen2.5-14B-Instruct-1M.
Twenty routing seeds are used.

\paragraph{Routing concentration.}
Across the 35 cells, the median deployed top-tool share is $0.65$, with range $0.39$--$1.00$.
Thirty cells route at least half their requests to a single tool.
The complete coalition lattice therefore contains many states the deployed system never visits.
The counterfactual analysis is intentionally interventional: it changes the available menu and reruns the driver under that alternative menu.

\paragraph{Agreement diagnostic.}
Cell-level Kendall agreement between BOHM and Shapley ranges from $-0.80$ to $+1.00$.
The nine cells where the deployed top pick is empirically best have mean agreement $+0.22$.
The remaining 26 have mean agreement approximately $+0.01$.
A flat five-tool hierarchy yields a gap of $+0.156$ in the same direction.
Restricting the analysis to the 26 cross-family driver/top-tool cells yields a larger gap of $+0.34$, so the main split is not explained by same-family preference.

\paragraph{Worked LiveCodeBench pair.}
Qwen3.6-A3B selects gpt-oss-120b on 45\% of deployed routes and that tool is empirically best on the 100-problem cell.
GLM-5.1-FP8 selects DeepSeek-V3.2 on 69\% of routes even though gpt-oss-120b is empirically stronger.
Table~\ref{tab:agentic_worked} gives the corresponding BOHM and Shapley values in the main text.

\section{Additional sensitivity details}
\label{app:sensitivity}

\subsection{Depth protocol}

Balanced synthetic trees use branching factor three and depths one through four, corresponding to 3, 9, 27, and 81 leaves.
Leaf qualities are linearly spaced.
The runs use 5,000, 20,000, 60,000, and 120,000 rounds respectively over ten seeds.
Final state--quality Kendall alignment is approximately $1.00$, $0.71$, $0.72$, and $0.67$.
The purpose of this experiment is to isolate the interaction between depth and routing exposure rather than to claim a universal depth law.

\subsection{Natural versus random LLM grouping}

The natural LiveCodeBench hierarchy is compared with ten independently shuffled three-tier assignments.
Each shuffled hierarchy is evaluated over the same 20 routing seeds.
The natural hierarchy produces Kendall alignment $0.739$ and tier-mass spread $0.279$.
The shuffled hierarchies average $0.507$ and $0.142$ respectively.
This experiment tests hierarchy sensitivity, not whether one hierarchy is intrinsically valid.

\subsection{Domain-specific tiering}

The five-domain comparison uses the same 18 models and the same routing hyperparameters, $\eta=0.05$ and $\epsilon=0.05$.
For each benchmark, a domain-specific hierarchy is constructed by reranking the models on that benchmark and regrouping them into the same $[3,3,2]$ shape.
The fixed condition keeps the original LiveCodeBench hierarchy unchanged.
The full results are reproduced in Table~\ref{tab:domain_tiering}.

The domain-specific result is strongest on HumanEval, where the fixed hierarchy gives $\tau=0.105$ and the retiered hierarchy gives $0.476$.
LiveCodeBench itself shows the expected near-tie in the opposite direction, $0.739$ fixed versus $0.715$ retiered, because the fixed hierarchy was originally constructed from LiveCodeBench.
The experiment therefore demonstrates sensitivity to domain structure rather than an automatic advantage for re-tiering.

\end{document}